\documentclass[letterpaper, 10 pt, conference]{ieeeconf}
\IEEEoverridecommandlockouts

\let\proof\relax
\let\endproof\relax
\usepackage{amsthm}
\usepackage[utf8]{inputenc}
\usepackage{amsmath,amssymb}
\usepackage{booktabs}
\usepackage{graphicx}
\usepackage{graphicx}
\usepackage{subcaption}
\usepackage[dvipsnames]{xcolor}
\usepackage{numprint}
\usepackage{csvsimple}
\usepackage{comment}
\usepackage[linesnumbered,ruled,noend]{algorithm2e}

\usepackage{tabularx,siunitx}
\usepackage{cite}
\usepackage{rotating} 
\usepackage{hyperref}
\usepackage[pscoord]{eso-pic} 

\usepackage[font=small]{caption}  
\usepackage{multirow}
 
\newtheorem{theorem}{Theorem}
\newtheorem{problem}{Problem}

\newtheorem{proposition}{Proposition}

\newtheorem{lemma}{Lemma}
\newtheorem{definition}{Definition}

\newcommand{\reals}{\mathbb{R}}
\newcommand{\N}{\mathcal{N}}

\newcommand{\naturals}{\mathbb{N}}

\newcommand{\tkdef}{ \forall t \in \big[t_k, t_{k+1} \big)}
\newcommand{\expBelief}{\textbf{b}}

\newcommand{\px}{\mathbf{x}}
\newcommand{\pz}{\mathbf{z}}
\newcommand{\pw}{\mathbf{w}}
\newcommand{\pv}{\mathbf{v}}

\usepackage[normalem]{ulem} 

\begin{document}

\title{\LARGE \bf Continuous-Time Gaussian Belief Trees for Motion Planning
\author{Rayan Mazouz$^{1\star}$, Qi Heng Ho$^{2\star}$, Zachary N. Sunberg$^{1}$, Morteza Lahijanian$^{1}$}
\thanks{$^{1}$Authors are with the Department of Aerospace Engineering Sciences at the University of Colorado Boulder, CO, USA
     {\tt\small \{\texttt{firstname}.\texttt{lastname}\}@colorado.edu}
        \\
\phantom{.} \hspace{0.3mm} $^{2}$Author is with the  Department of Aerospace and Ocean Engineering at Virginia Tech, Blacksburg, VA, USA
    {\tt\small \{\texttt{qihengho}\}@vt.edu}
\\
$^{\star}$ Equal contributions.}
}

\maketitle

\begin{abstract}
We address sampling-based motion planning for continuous-time stochastic systems under process and measurement uncertainty, with probabilistic guarantees on safety and performance. The robot dynamics are modeled as a continuous-time linear stochastic differential equation, while sensor measurements arrive at discrete time instants. 
We derive an offline hybrid belief propagation model in which the belief evolves according to continuous-time ODEs between measurements and undergoes discrete Kalman filter update jumps at measurement times. 
To ensure safety, we introduce a belief-barrier-function-based safety checker for segment-level probabilistic verification. This enables the planner to certify safety over entire continuous trajectory segments and detect inter-sample chance-constraint violations that are missed by conventional node-based checks. Together, these components provide a principled framework for sampling-based belief planning that accounts for both continuous-time uncertainty propagation and continuous-time safety requirements. 
We integrate the method with RRT and SST planners and evaluate it across multiple benchmark environments. 
The results show that the proposed method achieves high success rates and robust enforcement of chance constraints, including in narrow-passage scenarios where discrete-time counterparts fail due to missed inter-sample unsafe behavior. 
\end{abstract}

\section{Introduction}







Sampling-based algorithms have recently emerged as powerful planning tools that can efficiently search large-dimensional spaces, thereby rapidly finding solutions to complex problems \cite{kavraki1996prm, lavelle1998rrt, est1997}.
They have enabled applications ranging from self-driving cars 
and surgical robots 
to autonomous UAVs.
While most existing planners assume access to perfect state information, real robotic systems must contend with uncertainty in both motion (e.g., from modeling errors) and sensing (e.g., from noisy measurements). This challenge is particularly pronounced when observation uncertainty makes it difficult to ensure safety while maintaining efficiency. In practice, robotic systems evolve in continuous time, yet sensors only provide discrete measurements, creating a mismatch between real-world dynamics and the simplified models often used in planning. Recent work \cite{ho2022gaussian} addresses this issue by formulating planning directly in the belief space, incorporating uncertainty into the decision-making process. However, this framework is rooted in a discrete-time setting, which forfeits guarantees when applied to the inherently continuous-time domain of real systems. In this work, we close this gap by extending sampling-based belief-space planning to \emph{continuous-time}.

Consider a drone navigating through an urban environment with narrow alleys and obstacles. The drone’s motion evolves continuously in time, while its onboard sensors (e.g., cameras, LIDAR, or GPS) provide measurements only at discrete intervals. If planning is performed solely in the discrete-time domain, safety is evaluated only at the discrete belief states $b_t$ and $b_{t+1}$. As illustrated in Fig.~\ref{fig:ct-gbt}, this is problematic since
the intermediate predicted belief may intersect with obstacles,  causing the true state distribution to overlap with obstacles even when the discrete-time checkpoints appear safe.

Decision-making under uncertainty can be described using the framework of a Partially Observable Markov Decision Process (POMDP)~\cite{kaelbling1998planning}. This is, in general, an intractable problem~\cite{papadimitriou1987complexity}.
For continuous-time systems, the belief dynamics evolve according to stochastic differential equations driven by process and observation noise, making exact solutions even more challenging.
Approximation methods for continuous state, control, and observation spaces have been proposed \cite{sunberg2018pomcpow,garg2019despotalpha}, but they do not scale well when the control space is large and they often lack 
{practical} guarantees
{for finite computation time}.

\begin{figure}[t!]
    \centering
    \includegraphics[width=0.95\linewidth]{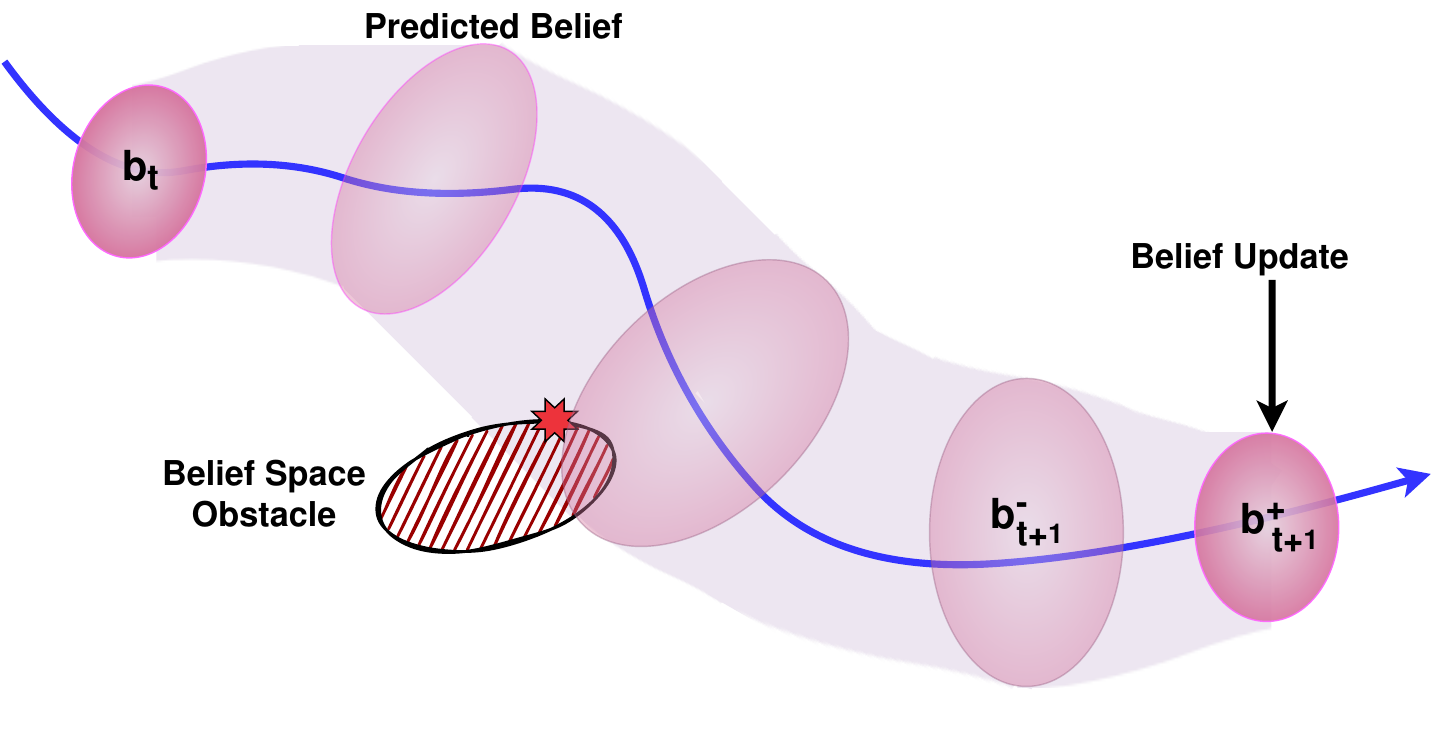}
    \caption{A pitfall in discrete-time, belief-space planning is to evaluate safety only at discrete time, when the measurements are received (i.e., at $b_t$ and $b^{+}_{t+1}$). This is insufficient because the state (and its belief) naturally grows in continuous-time between discrete-time belief updates due to process and measurement uncertainty, potentially intersecting with obstacles. 
    }
    \vspace{-5mm}
    \label{fig:ct-gbt}
\end{figure}

In recent years, sampling-based algorithms have been extended to account for motion uncertainty with probabilistic guarantees~\cite{cc-rrt,blackmore2011,Pairet:TASE:2021,burgard2008smr,Luna:AAAI:2014}. Chance-constrained tree-based planners have been shown to be effective due to their flexibility~\cite{cc-rrt,blackmore2011,Pairet:TASE:2021}. More recently, several works have extended these frameworks to explicitly address systems with measurement uncertainty~\cite{agha2014firm,Shan2017brm,Prentice2019brm, ho2022gaussian}. However, they are restricted to discrete-time dynamics and measurements. Simply reducing the discretization step significantly increases planning time, while measurement sampling rates are limited by 
hardware.


A principled method of safety analysis for dynamical systems
is \emph{barrier functions}. 
For verification, barrier certificates establish probabilistic safety guarantees in systems subject to process noise \cite{SANTOYO2021109439, mazouz2022safety, mazouz2024data, mazouz2024piecewise}. For control, control barrier approaches ensure safe operation under stochastic disturbances \cite{vahs2023belief, jagtap2020formal, mazouz2025piecewise} and explicitly account for measurement uncertainty \cite{laouar2024feasibility}. In motion planning, barriers facilitate constraint satisfaction in sampling-based algorithms \cite{yang2019sampling}. However, existing works typically address process noise, measurement noise, or deterministic constraints in isolation. To date, no approach integrates barrier functions with sampling-based continuous-time, belief-space planning.

Extending belief space planning in such settings to continuous time raises two main challenges. First, the belief dynamics are hybrid: covariance grows continuously between sensor updates and drops at each measurement via a Kalman Filter correction. Propagating this before measurements are available requires closed-form ODEs for both the estimation error covariance and the covariance of the estimate. Second, safety must hold for all $t$ in a continuous interval, not just at discrete nodes, so the collision checker must certify an entire trajectory segment.
 
We address both challenges by extending the GBT framework~\cite{ho2022gaussian}. First, we derive a hybrid belief propagation model in which the belief evolves according to continuous-time ODEs between measurement epochs and undergoes discrete Kalman filter updates at measurement times. Second, we adapt belief barrier functions~\cite{vahs2023belief} as a continuous-time collision checker for belief-space planning. Rather than verifying safety only at discrete planning nodes, the proposed checker reasons over entire continuous-time trajectory segments, thereby reducing missed inter-sample violations while remaining compatible with sampling-based search. Finally, we incorporate probabilistic control constraints into the same framework. Since the closed-loop control input depends on the uncertain state estimate, the control itself is stochastic and must also satisfy chance constraints, a formulation absent in prior GBT methods.

Specifically, we contribute a continuous-time GBT planning framework with:
\begin{itemize}
    \item a hybrid belief propagation model consisting of continuous-time ODEs for the estimation and estimate covariances, coupled with discrete Kalman filter updates at measurement epochs;
    
    \item a continuous-time probabilistic collision checker based on belief barrier functions, implemented through dense segment-level verification;
    
    \item the incorporation of probabilistic control constraints within the belief-space planning framework; and
    
    \item extensive benchmarking against discrete-time GBT variants, demonstrating the necessity and benefits of continuous-time belief propagation and segment-level safety checking.
\end{itemize}

\section{Problem Formulation}
\label{sec:problem}

We consider a motion planning problem for a robotic system subject to uncertainty in both motion and measurements, while requiring formal safety guarantees.

\subsection{Uncertain Dynamics and Measurements}

The noisy motion of the robot is described by (continuous-time) linear Stochastic Differential Equation (SDE)
\begin{equation}
        \label{eq:system}
        d\px(t) = \bigl(A \px(t) + B u(t)\bigr) dt + G\,d\pw(t),
\end{equation}
where $t \in \reals_{\geq 0}$ denotes \emph{continuous} time, $\px(t) \in \mathcal{X} \subset \mathbb{R}^{n}$ the state, 
$u(t) \in \mathcal{U} \subset \mathbb{R}^{m}$ the control input, where $\mathcal{U}$ is a compact convex polytope, and $\pw(\cdot)$ an $r$-dimensional Wiener process
with intensity matrix $Q \in \mathbb{R}^{r\times r}$. Further, $A \in \mathbb{R}^{n \times n}$, $B \in \mathbb{R}^{n \times m}$, and $G \in \mathbb{R}^{n \times r}$ are matrices with appropriate dimensions. 
We assume the initial state $\px(0)$ of the robot is uncertain and given as a Gaussian distribution with mean $\bar{x}(0) \in \mathcal{X}$ and covariance $\Sigma(0) \in \reals^{n \times n}$, i.e., $\px(0) \sim \N(\bar{x}(0), \Sigma(0))$.

In real-world settings, while the motion of the robot is continuous in time, 
its sensors provide noisy measurements only at \emph{discrete} intervals with sampling frequency $\tfrac{1}{\Delta t}$, where $\Delta t \in \mathbb{R}_{>0}$ is referred to as the discretization time.
Hence, the measurement model we consider is
\begin{align}
    \label{eq:measurement}
    \pz_k &= C \px_k + \pv_k,
\end{align}
where $\pz_k \in \mathbb{R}^p$ denotes the measurement at \emph{discrete} time step $k \in \naturals$, corresponding to continuous time $t = k \Delta t$, 
$C \in \mathbb{R}^{p \times n}$ is the measurement matrix, 
and $\pv_k$ is a zero-mean Gaussian (measurement) noise with covariance $R \in \reals^{p \times p}$, i.e., $\pv_k \sim \mathcal{N}\!\big(0,\, R\big)$. 
For clarity, we use $t_k = k \Delta t$ to denote the continuous time at time step $k$.
Also, we use variables with subscript $k$ to denote their dependence on discrete time steps and parentheses $t$ to denote continuous time, i.e., $\px_k = \px(t_k)$. Finally, we assume that matrices $A$, $B$, and $C$ are such that the system is  \emph{controllable} and \emph{observable}.

\subsection{State Evolution}
At execution time, the robot uses a Kalman Filter (KF) to obtain its state estimate $\hat{\px}_k$ and evolves under a stabilizing feedback controller 
\begin{align}
    \label{eq: feedback controller}
    \mathbf{u}(t) = \check{u}_k - K \big (\hat{\px}_k(t) - \check{x}_k(t)
    \big ), \quad k = \Big\lfloor \frac{t}{\Delta t} \Big\rfloor,
\end{align}
where $\check{u}_k \in \mathcal{U}$ is the piecewise-constant nominal feedforward control input, $K \in \reals^{m \times n}$ is the feedback gain, and $\check{x}_k(t)$ is the nominal state from the motion plan. The feedforward term $\check{u}_k$ is held constant between measurement epochs (it can only be updated when a new measurement $\mathbf{z}_k$ arrives at $t_k$), while the feedback term $K(\hat{\px}_k(t) - \check{x}_k(t))$ operates continuously on the current estimate $\hat{\px}_k(t)$.

For motion planning purposes, we use the SDE dynamics in~\eqref{eq:system} to obtain a nominal trajectory and then reason about the deviations from it under the controller in~\eqref{eq: feedback controller} using the uncertainty in both SDE and measurement models in~\eqref{eq:measurement}.
The solution of the SDE in~\eqref{eq:system} can be characterized as a continuous-time Gaussian process \cite{sarkka2019applied} with mean $\bar{x}(t)$, governed by
\begin{align}
    \label{eq: mean ODE}
    \dot{\bar{x}}(t) = A\bar{x}(t) + B\bar{u}(t), 
\end{align}
where $\bar{u}(t) = \check{u}_k$ for $k = \Big\lfloor \frac{t}{\Delta t} \Big\rfloor.$

To define a motion plan, let $\check{U} = \check{u}_0 \ldots \check{u}_{\mathcal{K}-1}$ for some $\mathcal{K} \in \naturals_{\geq 1}$, where $\check{u}_k \in \mathcal{U}$, be a control sequence (trajectory).  
By applying $\check{U}$ to \eqref{eq: mean ODE} from initial state $\bar{x}(0)$, we obtain a nominal trajectory $\check{X} = \check{x}_0 \ldots \check{x}_\mathcal{K}$ such that $\check{x}_0 = \bar{x}(0)$ and $\check{x}_{k+1} = \check{x}_k + \int_{t_k}^{t_{k+1}} (A\bar{x}(t) + B \check{u}_k) dt$. 
Then, we define a \textit{motion plan} to be the pair $(\check{U}, \check{X})$.

Given a motion plan $(\check{U}, \check{X})$, the evolution of $\px(t)$ under the controller in~\eqref{eq: feedback controller} is uncertain due to both motion and measurement uncertainty. Thus, $\px(t)$ is more appropriately represented by its \textit{belief} (i.e., probability density function) $b(t)$, i.e., $\px(t) \sim b(t)$. This belief $b(t)$ evolves as a piecewise-continuous process: $\tkdef$, it follows the SDE in~\eqref{eq:system}, while at $t = t_{k+1}$, upon receiving a new measurement, it incurs a discrete update determined by the KF. 

Let $\mathcal{X}_\mathrm{obs} = \{\mathcal{X}^o_i, \ldots, \mathcal{X}^o_N\}$, where $\mathcal{X}^o_i \subset \mathcal{X}$, denotes the set of unsafe (obstacle) regions, including workspace obstacles and out-of-bound velocities.  Each obstacle region is assumed to be a convex polytope, defined by the intersection of finitely many half-spaces, i.e.,
\begin{align}
    \label{eq:obstacles}
    \mathcal{X}^o_i
    = \bigcap_{q=1}^{n_i}
    \left\{ x \in \mathbb{R}^n \mid \alpha_{i,q}^\top x \;\geq\; \gamma_{i,q} \right\},
\end{align}
for $n_i \geq n+1$ collection of vectors $\alpha_{i,q} \in \mathbb{R}^n$ and scalars $\gamma_{i,q} \in \mathbb{R}$. 
Similarly, let $\mathcal{X}_\mathrm{goal} \subseteq \mathcal{X}$ be the goal region that the robot must reach, which is also a convex polytope. 


Similarly, the compact set $\mathcal{U}$ can be represented as the intersection of finitely many half-spaces
\begin{align}
\mathcal{U}
= \bigcap_{r=1}^{n_u}
\left\{ u \in \mathbb{R}^m \mid \beta_r^\top u \;\geq\; \eta_r \right\},
\end{align}
where $n_u \geq m+1$, with $\beta_r \in \mathbb{R}^m$ and $\eta_r \in \mathbb{R}$ defining the linear control constraints.

We are interested in finding a motion plan that avoids $\mathcal{X}_\mathrm{obs}$ and reaches $\mathcal{X}_\mathrm{goal}$.
However, since the robot state is uncertain, these requirements must be specified probabilistically.
Particularly, we aim to bound the probability of the robot lying in a set $\mathcal{X}_j$ with $j \in \{\mathrm{obs}, \mathrm{goal}\}$ at time $t \in [0,t_\mathcal{K})$, which is given by
\begin{align}
    \label{eq:state-prob}
    p\big(\px(t) \in \mathcal{X}_j\big) \;=\; \int_{\mathcal{X}_j} b(t)(y)\,dy.
\end{align}
Moreover, the closed-loop control input $\mathbf{u}(t)$ in~\eqref{eq: feedback controller} is stochastic, since it depends on the uncertain state estimate $\hat{\px}_k(t)$.  Hence, we also need to bound the probability that $\mathbf{u}$ is not admissible during execution, i.e., $p(\mathbf{u}(t) \notin \mathcal{U})$.


The problem we consider in this work is as follows.

\begin{problem}
\label{problem1}
Given a robot with noisy dynamics as SDE in~\eqref{eq:system} and measurement model in~\eqref{eq:measurement} with initial state $\px(0) \sim \mathcal{N}(\bar{x}(0), \Sigma(0))$, 
a convex polytopic control set $\mathcal{U}$,
a set of convex, polytopic unsafe regions $\mathcal{X}_{\mathrm{obs}} \subset \mathcal{X}$, a convex polytopic goal region $\mathcal{X}_{\mathrm{goal}} \subseteq \mathcal{X}$, a risk probability threshold $P_{\mathrm{safe}} \in (0,1)$, find a motion plan $(\check{U}=\check{u}_0\ldots \check{u}_{\mathcal{K}-1}, \check{X}=\check{x}_0\ldots \check{x}_\mathcal{K})$ such that
\begin{subequations}
    \label{eq:chance constraints}
    \begin{align}
        \label{eq:constraint1}
        &\!\! p\big(\px(t) \in \mathcal{X}_{\mathrm{obs}}\big) + p\big(\mathbf{u}(t) \notin \mathcal{U}\big) < P_{\mathrm{safe}},
        && \forall t \in [0, t_{\mathcal{K} }], \\
        \label{eq:constraint2}
        &\!\! p\big(\px(t) \in \mathcal{X}_{\mathrm{goal}}\big) \geq 1- P_{\mathrm{safe}} && \exists t \in [0, t_{\mathcal{K} }].
    \end{align}
\end{subequations}
\end{problem}
Note that this problem is challenging for two main reasons: (i) motion planning must be performed without realized measurements at planning time, and more importantly, (ii) the chance constraints in \eqref{eq:chance constraints} are specified over continuous time, whereas measurements and control inputs are in discrete time. Typical belief space motion planning methods are not able to natively address these challenges, as they typically perform validity-checking only at discrete nodes and lack mechanisms to ensure continuous-time chance constraints are satisfied.  
Our approach overcomes challenge~(i) by propagating an \emph{expected belief} in place of unavailable future measurements, and addresses challenge~(ii) by employing \emph{barrier-based reasoning} to certify safety along continuous trajectory segments.  


\section{Preliminaries}
\label{sec:preliminaries}

In this section, we briefly review the discrete-time offline Kalman filter formulation used in belief-space planning and Gaussian Belief Trees (GBT), a class of sampling-based planners for motion planning under uncertainty \cite{ho2022gaussian}.

\subsection{Discrete-time Offline Kalman Filter}

The Kalman Filter (KF) provides an optimal recursive state estimator under linear, Gaussian assumptions during online execution. During planning, however, measurements are not available a priori. Instead, we need to account for uncertainty in all the possible state estimates that could be realized during execution, which we call the offline KF \cite{bry2011rapidly}. 

Work \cite{bry2011rapidly} demonstrates that for a discrete-time linear system:
\begin{equation}
\label{eq:discrete}
x_{k+1} = F x_k + H u_k + w_k,
\end{equation}
paired with a linear measurement model in \eqref{eq:measurement},
we can predict the evolution of the system state's mean and covariance (the belief state) before any actual measurements are taken. Here, at time step $k$, $x_k$ is the state, $u_k$ is the control input, and $F$ and $H$ are the state transition and control matrices, respectively. The process noise $w_k$ is assumed to be i.i.d. white Gaussian noise such that $w_k \sim \mathcal{N}(0, Q_k)$.

When a stabilizing feedback controller in~\eqref{eq: feedback controller} is applied and the current 
state estimate $\hat{x}_k$ is available, the closed-loop system evolves according to
\begin{align}
   \px_{k+1} = F x_k + H \big(\check{u}_k - K(\hat{x}_k - \check{x}_k)\big) + w_k. 
\end{align} 
Since the dynamics and measurements models are linear and the initial state is Gaussian, the belief is also Gaussian. 


Let $b(\px_k \mid x_0, \mathbf{z}_{0:k})$ denote the belief over the state at time $t_k$, conditioned on the specific measurement history encountered during a single execution.
Because future measurements are unknown during planning, we work with the \emph{expected belief}, defined by marginalizing over all possible measurements~\cite{theurkauf2023chance}:
\begin{align}
    \label{eq:exp-belief}
    \expBelief(\px_k) 
    &= \mathbb{E}\!\left[\, b\!\left(\px_k \,\middle|\, x_0, \mathbf{z}_{0:k}\right) \,\right] = \int b(\textbf{x}_{k}\mid x_{0}, \textbf{z}_{0:k}) \; p(\textbf{z}_{0:k}) \; d \textbf{z}, 
\end{align}

This representation leads to two distributions that characterize the uncertainty. First, the online posterior $p(x_k \mid \hat{x}_k) = \mathcal{N}(\hat{x}_k, \Sigma_k)$ represents the robot's internal belief during execution, where $\Sigma_k$ is the estimation error covariance. Second, the distribution of state estimate $p(\hat{x}_k) = \mathcal{N}(\check{x}_k, \Lambda_k)$ characterizes the estimate itself as a random variable; from a planning perspective, the estimate $\hat{x}_k$ fluctuates around the nominal path $\check{x}_k$ with covariance $\Lambda_k$ due to sensor noise and tracking errors. This gives us the total expected belief $\expBelief(\px_k) = \mathcal{N}(\check{x}_k, \Gamma_k)$ which captures the total uncertainty of the true state from the planning perspective. The total covariance is defined as the sum of the estimation and estimate covariances, $\Gamma_k = \Sigma_k + \Lambda_k$. These terms are updated recursively at discrete time steps~\cite{bry2011rapidly}:
\begin{subequations}
\begin{align}
& \Sigma_k^{-}  = F \, \Sigma_{k-1}^{+} F^{\top} + Q_d, \label{eq:discreteKF1} \\
& L_k = \Sigma_k^{-} C^{\top} \big(C \Sigma_k^{-} C^{\top} + R\big)^{-1}, \label{eq:discreteKF2} \\
& \Sigma_k^{+} = \Sigma_k^{-} - L_k C \Sigma_k^{-}, \label{eq:discreteKF3} \\
& \Lambda_k^{+} = (F - H K)\Lambda_{k-1}^{+}(F - H K)^{\top} + L_k C \Sigma_k^{-}. \label{eq:discreteKF4} 
\end{align}
\end{subequations}

\subsection{Gaussian Belief Trees}

Gaussian Belief Trees (GBT) extend sampling-based motion planning methods such as Rapidly-exploring Random Trees (RRT) and Stable Sparse-RRT (SST) to partially observable systems by planning directly in the Gaussian belief space. GBT incrementally constructs a tree in the belief space by alternating between belief sampling, nearest-neighbor selection, belief propagation, and collision checking. At each iteration, a random Gaussian belief sample is generated and the nearest node in the tree is selected according to a belief-space distance metric such as the Wasserstein distance metric. A nominal control input is then applied for a finite duration, and the resulting belief is propagated using the system dynamics and uncertainty model. If the propagated belief satisfies the discrete-time chance constraints, the new node and trajectory segment are added to the tree. An RRT instantiation of GBT is presented in Alg.~\ref{alg:gbt}.

\section{Continuous-Time Gaussian Belief Trees}
\label{sec:CT-GBT}

Existing GBT planners are formulated for strictly discrete-time belief dynamics and perform safety verification only at discrete propagation nodes. In this section, we present our approach to Problem~\ref{problem1} by extending GBT to continuous-time systems. We present continuous-time belief propagation and continuous-time collision checking, enabling detection of inter-sample safety violations that are missed by discrete-time methods. Finally, we prove the correctness of our approach and the algorithm is probabilistically complete. 


\subsection{Algorithm}
Alg.~\ref{alg:gbt} presents our GBT-RRT planner for continuous-time belief space dynamics. Alg.~\ref{alg:gbt} has the same sampling-based tree-search structure as discrete-time GBT-RRT.
\begin{algorithm}[b!]
    \SetKwInOut{Input}{Input}
    \SetKwInOut{Output}{Output}
    \caption{\tt GBT-RRT($b_{init}, \mathcal{X}_{goal}, k, \Delta t, \mathcal{U}$)}
    \label{alg:gbt}
    \Output{Belief tree $\mathcal{T}$ rooted at $b_{init}$}
    $\mathcal{T} \gets (\mathbb{V} \leftarrow \{b_{init}\}, \ \mathbb{E} \leftarrow \emptyset)$ \\
    \For{$i \gets 1$ \KwTo $k$}{
        $b_{rand} \gets \texttt{SampleBelief}()$ \\
        $b_{near} \gets \texttt{NearestNeighbor}(b_{rand}, \mathcal{T})$ \\
        ${ \tau \gets \texttt{SampleDuration}(0, \Delta t)}$ \\
        $u \gets \texttt{SampleControlInput}(\mathcal{U})$ \\
        ${\color{blue}b_{new} \gets \texttt{PropagateBelief}(b_{near}, u, \tau)}$ \\
        \uIf{{\color{blue}\texttt{isCollisionFreeBelief}($b_{near}, b_{new}$)}}{
            $\mathbb{V} \gets \mathbb{V} \cup \{b_{new}\}$ \\
            $\mathbb{E} \gets \mathbb{E} \cup \{\texttt{edge}(b_{near}, b_{new})\}$ \\
        }
    }
    \Return $\mathcal{T} = (\mathbb{V}, \mathbb{E})$
\end{algorithm}
There are two main distinctions from its discrete-time counterpart, highlighted as follows.
\begin{itemize}

    \item Line~7: $b_{new} \gets \texttt{PropagateBelief}()$,  
    which propagates beliefs according to continuous-time dynamics under feedback control.  

    \item Line~8: \texttt{isCollisionFreeBelief}$()$,  
    which validates safety along the entire continuous trajectory segment between measurement time steps.
\end{itemize}



\subsection{Belief Propagation}

Let the expected belief be as defined in~\eqref{eq:exp-belief}. 
During offline motion planning, we use this expected belief for the evaluation of the chance constraints in~\eqref{eq:constraint1} and \eqref{eq:constraint2}. Given an initial belief $b(x_0)$, a nominal trajectory $\check{X}$, and a nominal control sequence $\check{U}$ paired with the feedback controller $\mathbf{u}(t)$ in ~\eqref{eq: feedback controller}, the expected belief is represented by $\expBelief(x(t))=\mathcal{N}(\bar{x}(t),\Gamma(t))$. 

The total covariance $\Gamma(t)=\Sigma(t)+\Lambda(t)$ comprises the internal estimation uncertainty  $\Sigma(t)$, and the covariance of the state estimate $\Lambda(t)$. The evolution of the mean and covariance components is characterized as follows.

\begin{theorem}
\label{thm: continuous dynamics}
The evolution of the expected belief $\expBelief(\px(t)) = \mathcal{N}(\bar{x}(t), \Gamma(t))$ 
for a system with continuous-time dynamics as in~\eqref{eq:system} 
is governed by the following system of ODEs:
\begin{subequations}
\label{eq:prop}
\begin{align}
\dot{\bar{x}}(t) &= A \bar{x}(t) + B \bar{u}(t), \label{eq:mean}\\
\dot{\Sigma}(t) &= A \Sigma(t) + \Sigma(t) A^\top + G Q G^\top,
\label{eq:process-cov}\\
    \dot{\Lambda}(t) &= 
    (A - BK)\Lambda(t) + \Lambda(t)(A - BK)^{\top}. \label{eq:measure-cov}
\end{align}
\end{subequations}
\end{theorem}
\begin{proof}
The mean~\eqref{eq:mean} and estimation covariance dynamics ~\eqref{eq:process-cov} follow from standard
moment equations for linear SDEs. Between measurement updates, the
estimator evolves in prediction mode and no measurement correction is
applied. It remains to characterize the covariance of the estimate. Let $
\delta(t) := \hat{x}(t)-\bar{x}(t)$.
Using the prediction dynamics of the estimator and the feedback law
\[
u(t)=\bar{u}(t)-K(\hat{x}(t)-\bar{x}(t)),
\]
we obtain
\[
\dot{\delta}(t) = (A-BK)\delta(t).
\]
Since $\Lambda(t)=\mathbb{E}[\delta(t)\delta(t)^\top]$, differentiating gives
\[
\dot{\Lambda}(t)
=
(A-BK)\Lambda(t)+\Lambda(t)(A-BK)^\top.
\]
Together with the standard mean and covariance equations for the linear
SDE, this yields the stated continuous-time belief dynamics.
\end{proof}

Theorem~\ref{thm: continuous dynamics} and Equations~\eqref{eq:discreteKF2}--\eqref{eq:discreteKF4} together define the hybrid belief dynamics for $\mathbf{b}$, where $t_k^-$ and $t_k^+$ denotes the belief state immediately after/before a measurement at $t_k$:
\begin{align*}
        &\dot{\bar{x}}(t) = A\bar{x}(t) + Bu(t), && t \in (t_k,\, t_{k+1}), \\
    &\dot{\Sigma}(t) = A\Sigma(t) + \Sigma(t)A^\top + GQG^\top, && t \in (t_k,\, t_{k+1}), \\
    &\dot{\Lambda}(t) = (A{-}BK)\Lambda(t) + \Lambda(t)(A{-}BK)^\top, && t \in (t_k,\, t_{k+1}), \\
    &\bar{x}(t_k^+) = \bar{x}(t_k^-), && t = t_k, \\
    &\Sigma(t_k^+) = \Sigma(t_k^-) - L_k C\,\Sigma(t_k^-), && t = t_k, \\
    &\Lambda(t_k^+) = \Lambda(t_k^-) + L_k C\,\Sigma(t_k^-), && t = t_k, 
\end{align*}
where Kalman gain $L_k = \Sigma(t_k^-)C^\top\bigl(C\Sigma(t_k^-)C^\top + R\bigr)^{-1}$.

\subsection{Validity Checking: Probabilistic Safety Certificates}
\label{sec:BBF}

To formally guarantee safety of the motion planner, we utilize belief space barrier functions as the validity checker of Alg.~\ref{alg:gbt}. We first discuss state constraints, followed by control constraints, and finally risk allocation.

\subsubsection{State constraints}

For each obstacle $\mathcal{X}^o_i$ and candidate nominal state $\check{x}(t) \notin \mathcal{X}^o_i$, we construct a separating half-space $\mathcal{S}_i(t) = \{x \in \mathbb{R}^n \mid \alpha_i(t)^\top x \geq \gamma_i(t)\}$ such that $\check{x}(t) \in \mathcal{S}_i(t)$ and $\mathcal{S}_i(t) \cap \mathcal{X}^o_i = \emptyset$, reducing probabilistic collision avoidance to a single linear chance constraint.

Let the state-space safety specification be in the form of half-spaces $\alpha^\top \px(t) \geq \gamma$. Then, at a given time $t$, the probability of satisfying this half-space constraint is \cite{zhu2019chance}
\begin{equation}
    \Pr[\alpha^\top \px(t) \geq \gamma]
    =
    \frac{1}{2}
    \left(
    1 +
    \mathrm{erf}\left(
    \frac{{\alpha}^\top \check{x}(t) - \gamma}
    {\sqrt{2 {\alpha}^\top (\Sigma(t) + \Lambda(t)) {\alpha}}}
    \right)
    \right),
    \label{eq:ProbHalfspace}
\end{equation}
    where $\check{x}(t)$ is the mean of the expected belief (the planned nominal state), $\Sigma(t)$ and $\Lambda(t)$ are the estimation and estimate covariances and $\mathrm{erf}(\cdot)$ denotes the error function.

Given an obstacle defined by a half-space constraint with a risk budget $\delta_j$, we define a belief barrier function~\cite{blackmore2011, vahs2023belief} as follows.

\begin{proposition}
Let the expected belief $\expBelief(t)$ be described by mean $\check{x}(t)$, estimation covariance $\Sigma(t)$, and estimate covariance $\Lambda(t)$.
Then, given a risk budget $\delta_j$, a chance constraint in the belief space is defined by
\begin{equation}
\label{eq:belief_constraint}
\begin{split}
h^j_x(b(t)) = &\alpha^\top \check{x}(t) - \gamma \\
&-\mathrm{erf}^{-1}\!\left(1 - 2\delta_j\right) \sqrt{2\,\alpha^\top(\Sigma(t) + \Lambda(t))\alpha},
\end{split}
\end{equation}
where $\Sigma(t) + \Lambda(t)$ represents the total system uncertainty.
For probabilistic safety, it must hold that $h^j_x(b(t)) \;\geq\; 0$.
\end{proposition}

We refer to~\eqref{eq:belief_constraint} as a belief barrier function.
In the planner, this function is used as a probabilistic collision
checker by evaluating whether $h_x^j(b(t))\geq 0$ along candidate
trajectory segments.

\subsubsection{Control Constraints}
Since the control input~\eqref{eq: feedback controller} is an affine function of the Gaussian estimate deviation $\hat{\px}_k(t) - \check{x}_k(t)$, each scalar control bound is a linear constraint on a Gaussian random variable. Since $\hat{\px}_k(t) - \check{x}_k(t)$ has covariance $\Lambda(t)$, the control input $\mathbf{u}(t)$ is Gaussian with mean $\check{u}_k$ and covariance $K\Lambda(t)K^\top$. The probability of satisfying any such bound is therefore given by a similar closed-form expression as~\eqref{eq:ProbHalfspace}, with $\alpha$ and $\gamma$ replaced by the corresponding control constraint parameters $\beta$ and $\eta$, and the covariance replaced by $K\Lambda(t)K^\top$.

\begin{proposition}
Let the expected belief $\expBelief(t)$ be described by mean $\check{x}(t)$, estimation covariance $\Sigma(t)$, and estimate covariance $\Lambda(t)$, and let the control input be given by~\eqref{eq: feedback controller}.
Then, given a risk budget $\delta_j$, a control chance constraint in the belief space is defined by
\begin{equation}
\label{eq:control_belief_constraint}
h^j_u(b(t)) = \beta^\top \check{u}_k - \eta 
- \mathrm{erf}^{-1}\!\left(1 - 2\delta_j\right) \sqrt{2\,\beta^\top K\Lambda(t)K^\top\beta}.
\end{equation}
For probabilistic safety, it must hold that $h^j_u(b(t)) \geq 0$.
\end{proposition}

\subsubsection{Risk Allocation}
Let $M_x$ and $M_u$ denote the number of scalar state and control chance constraints, respectively, with $M = M_x + M_u$. Using Boole's inequality~\cite{blackmore2011}, we allocate individual risk budgets $\delta_j$ satisfying $\sum_{j=1}^{M} \delta_j \leq P_{\mathrm{safe}}$. In our implementation, we use the uniform allocation $\delta_j = P_{\mathrm{safe}} / M$. The following theorem gives a sufficient continuous-time condition under which the chance constraint remains satisfied over an entire segment.

\begin{theorem}[Sufficient condition for continuous-time probabilistic safety]
    \label{thm:safe}
    Let $\mathcal{H} = \{ h^j_x \}_{j=1}^{M_x} \cup \{ h^j_u \}_{j=1}^{M_u}$ denote the set of state and control belief barrier functions for System~\eqref{eq:system}. Suppose a control input $\mathbf{u}(t)$ with locally Lipschitz sample paths is applied. Then, if for all $h \in \mathcal{H}$ and $\tkdef$, the following holds
\begin{equation}
    \label{eq:safety-condition}
        \frac{\partial h}{\partial b}
        \dot{b}(t)
        \;\geq\; -\,h(b(t)),
\end{equation} 
where $\dot{b}(t)$ is described by~\eqref{eq:mean},~\eqref{eq:process-cov} and~\eqref{eq:measure-cov}, then the closed-loop system is probabilistically safe over continuous-time interval $[t_k, t_{k+1})$.
\end{theorem}
\begin{proof}
Assume that $h(b(t_k))\geq 0$. From~\eqref{eq:safety-condition},
we have
\begin{align*}
        \dot{h}(b(t)) + h(b(t)) \geq 0
    \quad \forall t\in[t_k,t_{k+1}).
\end{align*}
By the comparison lemma, this implies
\begin{align*}
        h(b(t)) \geq h(b(t_k))e^{-(t-t_k)} \geq 0,
    \quad \forall t\in[t_k,t_{k+1}).
\end{align*}
Therefore, each individual chance constraint holds for all $t\in[t_k,t_{k+1})$. Applying this argument to all $M$ barrier functions in $\mathcal{H}$ and using Boole's inequality yields the joint probabilistic safety condition.
\end{proof}

Theorem~\ref{thm:safe} establishes a sufficient condition for probabilistic
safety over the \emph{continuous-time} interval $[t_k, t_{k+1}]$. In practice, we verify~\eqref{eq:safety-condition} for each barrier function on a uniform grid with resolution $\delta t \ll \Delta t$, evaluating the condition at each intermediate point. This yields an exact certificate at the grid points and an approximate certificate between them, with worst-case inter-point error $\mathcal{O}(L^k_j\, \delta t)$, where $L^k_j$ is the Lipschitz constant of $h(b(\cdot))$ on $[t_k, t_{k+1}]$. This constant is finite, since $\check{x}(t)$, $\Sigma(t)$, and $\Lambda(t)$ satisfy linear ODEs with bounded right-hand sides on a compact interval. The resolution $\delta t$ can therefore be chosen to drive this error below any desired tolerance. The evaluation of~\eqref{eq:belief_constraint} and~\eqref{eq:control_belief_constraint} at each grid point is computationally efficient, as both the belief dynamics (Theorem~\ref{thm: continuous dynamics}) and the gradient of $h$ are available in closed form. Additionally, an exact continuous-time certificate can be recovered if the Lipschitz constants $L^k_j$ are known; we leave this extension to future work.
\subsection{Analysis}

Here, we discuss the soundness and probabilistic completeness of our algorithm for continuous-time belief space planning.
First, we equip the space of Gaussian beliefs with a suitable metric, such as the 2-Wasserstein distance. We assume that there exists a valid sequence of trajectories such that (i) the collision checker in \ref{sec:BBF} certifies the sequence as probabilistically safe, (ii) the sequence has non-zero clearance in the belief space and (iii) each trajectory segment is piecewise constant over durations supported by the duration sampler.


\begin{definition}[Collision Checker Clearance]
    Let a valid trajectory $b(t)$ be the output of a continuous map $\pi: [0,t_{\pi}] \rightarrow B$, such that $\pi(0)=b_0$ and $p(\pi(t_{\pi}) \in \mathcal{X}_{goal}) \geq 1-P_{safe}$. The clearance of $\pi$ is the maximal $\delta_{clear}$ such that, for all belief points in a ball of radius $\delta_{clear}$ centered on $\pi(t)$, $p\big(\px(t) \in \mathcal{X}_{obs}\big) + p\big(\mathbf{u}(t) \notin \mathcal{U}\big) < P_{\mathrm{safe}}$, $\forall t \in [0,t_{\pi}]$.
\end{definition}

\begin{lemma}[Correctness]
    Given System~\eqref{eq:system}, if Algorithm~\ref{alg:gbt} terminates with a motion plan, the motion plan satisfies constraints \eqref{eq:constraint1} and \eqref{eq:constraint2}.
\end{lemma}
\begin{proof}
    This follows directly from Theorem~\ref{thm:safe} and the fact that Algorithm~\ref{alg:gbt} only adds nodes that are probabilistically safe throughout the continuous time trajectory.
\end{proof}

\begin{theorem}[Probabilistic Completeness w.r.t. Conservatism]
    Assume there exists a valid belief trajectory such that the collision checker in \ref{sec:BBF} certifies it as probabilistically safe. As the number of iterations $k$
    approaches $\infty$, the probability of Algorithm~\ref{alg:gbt} finding a solution approaches $1$.
\end{theorem}

\begin{proof}
The result follows from the probabilistic completeness of kinodynamic RRT with forward propagation and random control inputs \cite{kleinbort2018probabilistic, ho2022gaussian}. Our setting differs in two ways: (i) the state space is a Gaussian space with a distance metric, making it a Riemannian manifold, and (ii) controls are sampled over discrete intervals of length $\Delta t$. Both modifications preserve the required conditions for probabilistic completeness: the Riemannian metric satisfies the properties of a metric space, and discretizing control intervals only restricts the space of allowed controls. Assuming that there exists a valid trajectory with the same restriction,  Alg.~\ref{alg:gbt} eventually selects the sequence of controls that generates the valid trajectory. Therefore, Algorithm~\ref{alg:gbt} retains probabilistic completeness: if a valid trajectory exists, the probability of finding it approaches $1$ as $k$ approaches $\infty$.
\end{proof}


\section{Evaluation}

To evaluate our approach, we study the performance on several benchmark cases. We specifically consider two continuous-time stochastic systems:
\begin{itemize}
    \item Single integrator (2D): The state $\px(t) \in \mathbb{R}^{2}$ evolves according to the SDE
    \begin{align}
        \label{eq:bench1}
        d\px(t) = u(t)\,dt + G\,d\pw(t),
    \end{align}
    i.e., $A = 0$, $B = I_2$ in~\eqref{eq:system}, with $u(t) \in \mathbb{R}^{2}$ the velocity control input.
    \item Double integrator (4D): The state $\px(t) = [\px_1(t)^\top, \px_2(t)^\top]^\top \in \mathbb{R}^{4}$, where $\px_1 \in \mathbb{R}^2$ is position and $\px_2 \in \mathbb{R}^2$ is velocity, evolves according to
        \begin{equation}
        \label{eq:bench2}
        d\begin{bmatrix}\px_1 \\ \px_2\end{bmatrix}
        =
        \begin{bmatrix}0 & I_2 \\ 0 & 0\end{bmatrix}
        \begin{bmatrix}\px_1 \\ \px_2\end{bmatrix}dt
        +
        \begin{bmatrix}0 \\ I_2\end{bmatrix}u(t)\,dt
        + G\,d\pw(t),
        \end{equation}
where $u(t) \in \mathbb{R}^{2}$ is the acceleration input.
\end{itemize}

\begin{figure}[t!]
    \centering
    \begin{subfigure}[b]{0.23\textwidth}
        \centering
        \includegraphics[width=\textwidth]{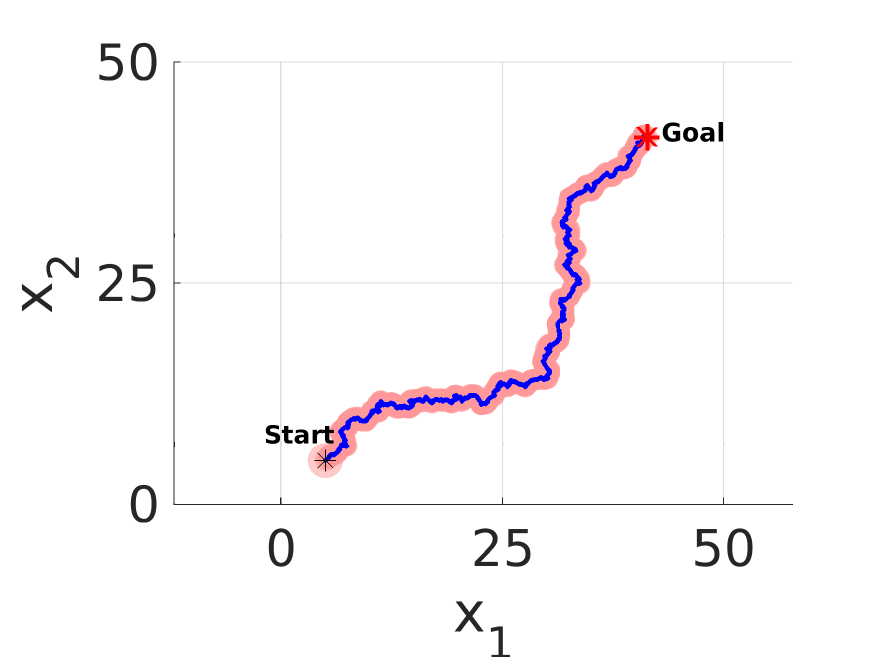}
        \caption{Environment 1}
        \label{fig:2D-env1}
    \end{subfigure}
    \begin{subfigure}[b]{0.23\textwidth}
        \centering
        \includegraphics[width=\textwidth]{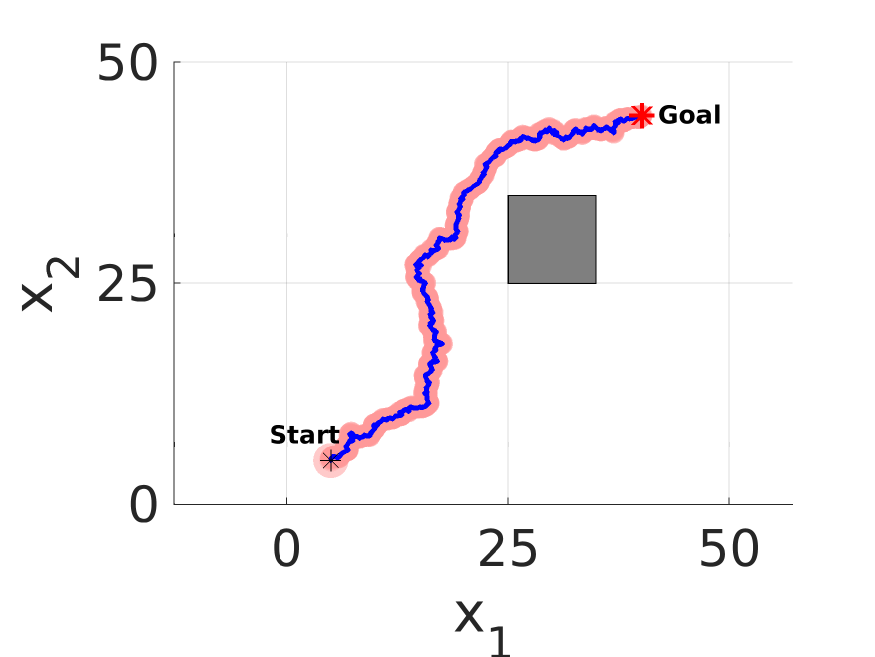}
        \caption{Environment 2}
        \label{fig:2D-env2}
    \end{subfigure}

    \begin{subfigure}[b]{0.23\textwidth}
        \centering
        \includegraphics[width=\textwidth]{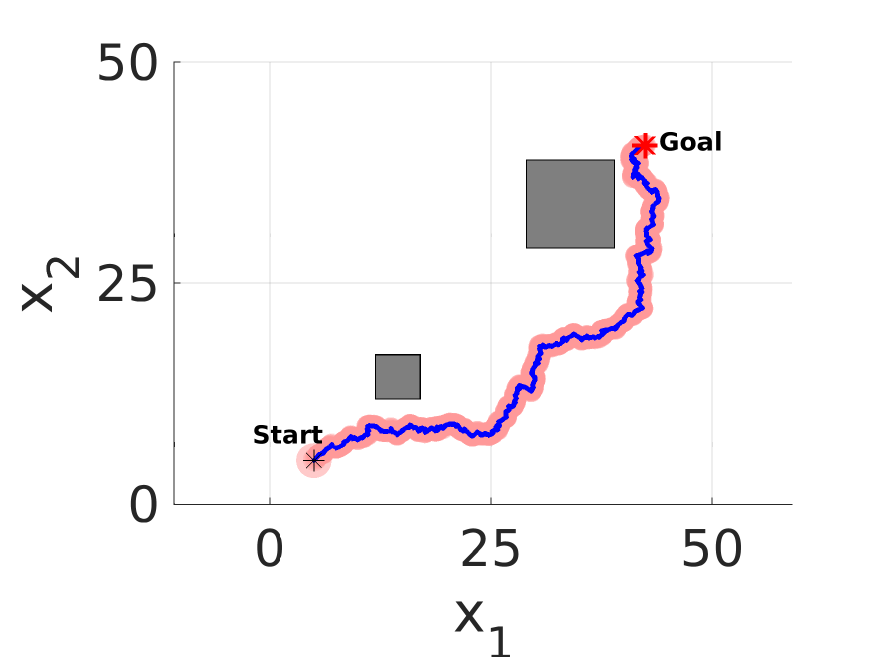}
        \caption{Environment 3}
        \label{fig:2D-env3}
    \end{subfigure}
    \begin{subfigure}[b]{0.23\textwidth}
        \centering
        \includegraphics[width=\textwidth]{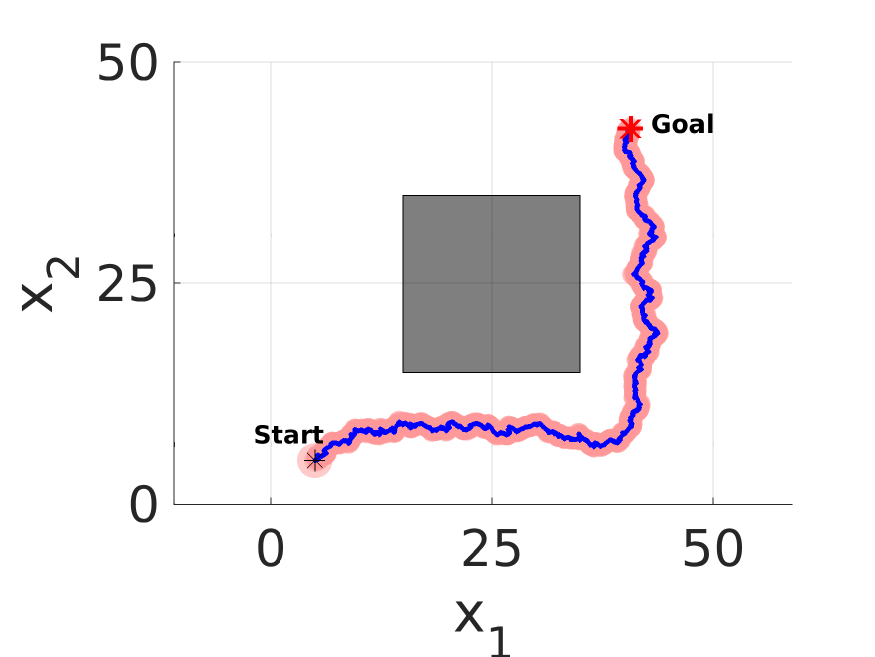}
        \caption{Environment 4}
        \label{fig:2D-env4}
    \end{subfigure}

    \begin{subfigure}[b]{0.23\textwidth}
        \centering
        \includegraphics[width=\textwidth]{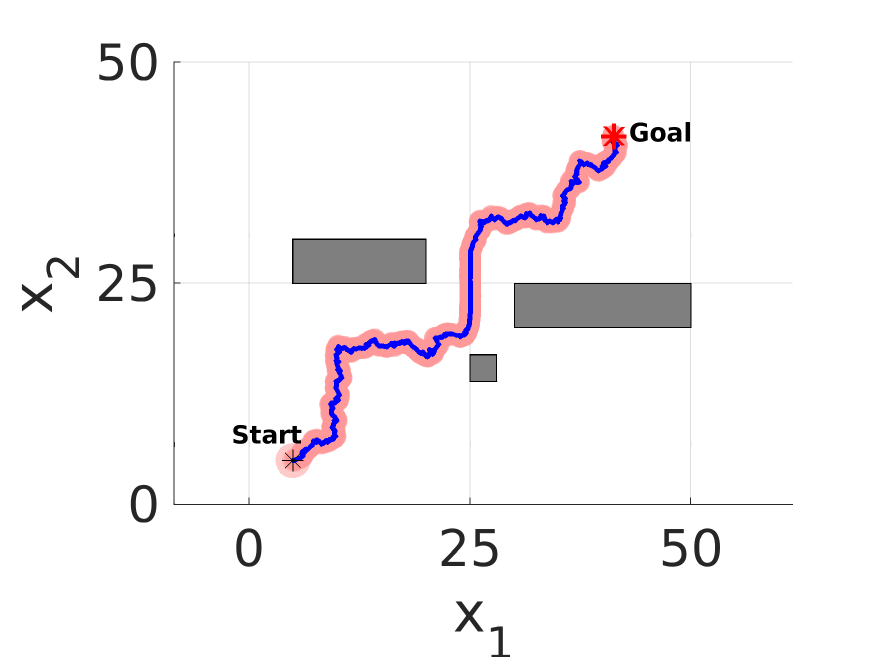}
        \caption{Environment 5}
        \label{fig:2D-env5}
    \end{subfigure}
    \begin{subfigure}[b]{0.23\textwidth}
        \centering
        \includegraphics[width=\textwidth]{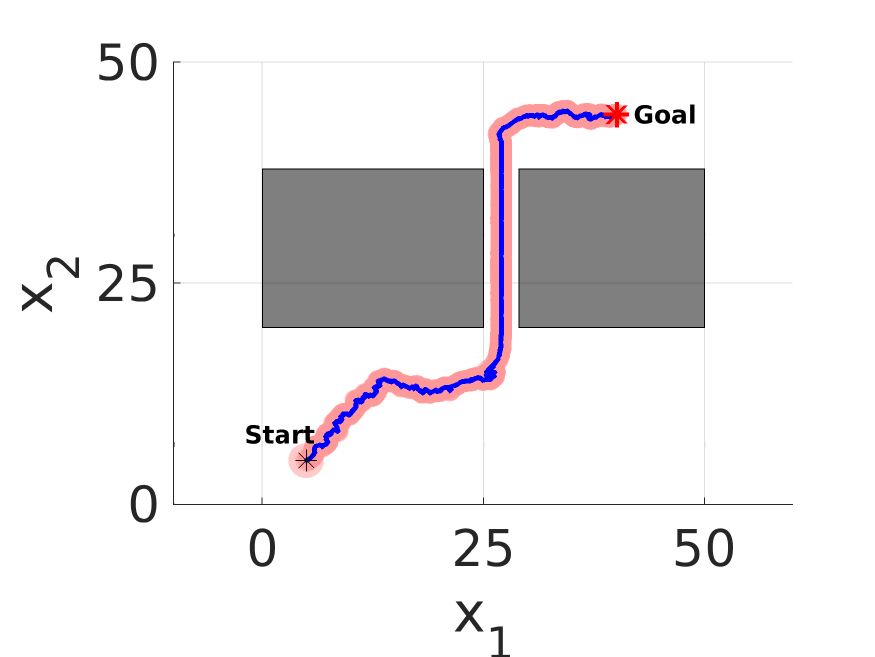}
        \caption{Environment 6}
        \label{fig:2D-env6}
    \end{subfigure}

    \caption{Continuous-time Gaussian belief space planner trajectories across six environments, where $X \in [0, 50]^2$.}
    \label{fig:2D-system}
    \vspace{-3.75mm}
\end{figure}

For both systems, the measurement model follows~\eqref{eq:measurement} with $C = I$ (full-state observation), measurement noise covariance $R = 0.1 I$, and sampling interval $\Delta t$. All the computations for the benchmarks were performed on a computer with 3.9 GHz 8-Core CPU and 128 GB of memory. We use a violation threshold of $P_{\mathrm{safe}} = 0.1$ for the experiments. 

\subsection{Continuous-Time Belief Trajectories}

Fig.~\ref{fig:2D-system} illustrates representative plans generated by the continuous-time RRT-based GBT across six benchmark environments. Here, the dynamics are governed by~\eqref{eq:bench1}, and various obstacle locations are selected to highlight the planner's ability to generate collision-free trajectories under uncertainty. Specifically, the results highlight the ability of the CT formulation to maintain safety in belief propagation while successfully navigating increasingly complex environments. For the cases in Figs.~\ref{fig:2D-env1} and \ref{fig:2D-env2}, the planner produces short, direct paths to the goal, while in more constrained scenarios (Figs.~\ref{fig:2D-env3}--\ref{fig:2D-env6}) it adapts by generating feasible but longer paths that respect the constraints and safety specifications.

\subsection{Benchmarks}

\begin{table*}[!ht]
    \centering
    \caption{\small Benchmark planner performance results for the \texttt{CT} and \texttt{DT} planners; using \texttt{RRT} and \texttt{SST}. Each entry is the mean of 100 simulations. The best (lowest) time and cost between \texttt{CT} and \texttt{DT} variants are bolded.}
    \label{tab:benchmarks}
    \begin{tabular}{l|cc|cc|cc|cc|cc|cc}
    \toprule
    & \multicolumn{2}{c|}{Env. 1} 
    & \multicolumn{2}{c|}{Env. 2} 
    & \multicolumn{2}{c|}{Env. 3}
    & \multicolumn{2}{c|}{Env. 4}
    & \multicolumn{2}{c|}{Env. 5}
    & \multicolumn{2}{c}{Env. 6}
    \\
    \texttt{Alg.}
    & Time (s) &  Cost
    & Time (s) &  Cost 
    & Time (s) &  Cost 
    & Time (s) &  Cost
    & Time (s) &  Cost
    & Time (s) &  Cost \\
    \midrule
     \multicolumn{13}{c}{\textbf{2D System}}\\
    \midrule
     CT-RRT   & 3.481 & \textbf{89.011} & \textbf{2.172} & \textbf{88.203} & \textbf{2.643} & \textbf{84.421} & 4.129 & \textbf{90.388} & 7.021 & \textbf{99.433} & \textbf{8.274} & \textbf{101.601} \\
     DT-RRT   & \textbf{2.589} & 89.731 & 3.533 & 96.712 & 2.864 & 86.918 & \textbf{3.754} & 91.422 & \textbf{6.671} & 101.781 & 9.012 & 103.941 \\
     CT-SST   & 8.042 & \textbf{80.311} & \textbf{6.341} & \textbf{83.901} & 7.612 & \textbf{81.102} & 8.401 & \textbf{79.522} & 9.488 & \textbf{85.442} & 10.941 & \textbf{88.287} \\
     DT-SST   & \textbf{6.701} & 81.598 & 7.154 & 85.377 & \textbf{6.621} & 82.412 & \textbf{7.384} & 80.944 & \textbf{9.247} & 86.052 & \textbf{10.083} & 89.812 \\
    \midrule
    
    \multicolumn{13}{c}{\textbf{Double Integrator}}\\
    \midrule
     CT-RRT   & 4.221 & 122.184 & \textbf{3.821} & \textbf{118.392} & 4.731 & \textbf{120.821} & 5.109 & \textbf{125.402} & 6.664 & \textbf{129.778} & 7.774 & \textbf{135.801} \\
     DT-RRT   & \textbf{3.471} & \textbf{124.512} & 4.378 & 121.704 & \textbf{3.984} & 123.311 & \textbf{4.621} & 127.244 & \textbf{6.089} & 133.917 & \textbf{7.281} & 137.641 \\
     CT-SST   & 8.619 & \textbf{110.012} & \textbf{7.732} & \textbf{113.447} & 8.231 & \textbf{111.633} & 9.084 & \textbf{109.401} & 10.117 & \textbf{115.912} & 11.652 & \textbf{118.511} \\
     DT-SST   & \textbf{7.492} & 112.331 & 8.224 & 115.589 & \textbf{7.781} & 113.482 & \textbf{8.592} & 111.203 & \textbf{9.433} & 117.112 & \textbf{11.083} & 120.921 \\
    \bottomrule

    \end{tabular}
    \vspace{-2mm}
\end{table*}

Next, the benchmarks for systems~\eqref{eq:bench1} and~\eqref{eq:bench2} are presented (Table~\ref{tab:benchmarks}), with inputs constrained to $U \in [-10, 10]^2$. 
We evaluate \emph{discrete-time (DT)} and \emph{continuous-time (CT)} GBT variants of the planners, comparing both \emph{planning time} and \emph{solution cost}.  Two algorithmic families are considered: \texttt{RRT} and \texttt{SST}. The results highlight the trade-offs between DT and CT formulations across multiple environments.

Across both systems, a consistent trade-off emerges between the CT and DT planners. The latter variants typically exhibit lower planning times, reflecting the reduced computational burden introduced by time discretization, whereas the CT more frequently achieve lower trajectory costs, which can be attributed, in part, to a more accurate representation of the underlying dynamics. This behavior is consistent across both RRT and SST planners and is observed in all environments.

The plots in Fig.~\ref{fig:barplots} further highlight a clear trade-off between the planners. For this benchmark, control is restricted to $U = [-1,1]^{2}$, which induces effects similar to those observed when decreasing the DT time-step. This setting forces longer horizons to reach the goal. As a result, differences in planning time between variants become less pronounced, with instances in which the CT planner is both faster and produces lower-cost trajectories. This indicates that reducing the DT time-step alone is insufficient to recover the same behavior observed in the CT planner.

\begin{figure}[b!]
    \centering
    \includegraphics[width=1.0\linewidth]{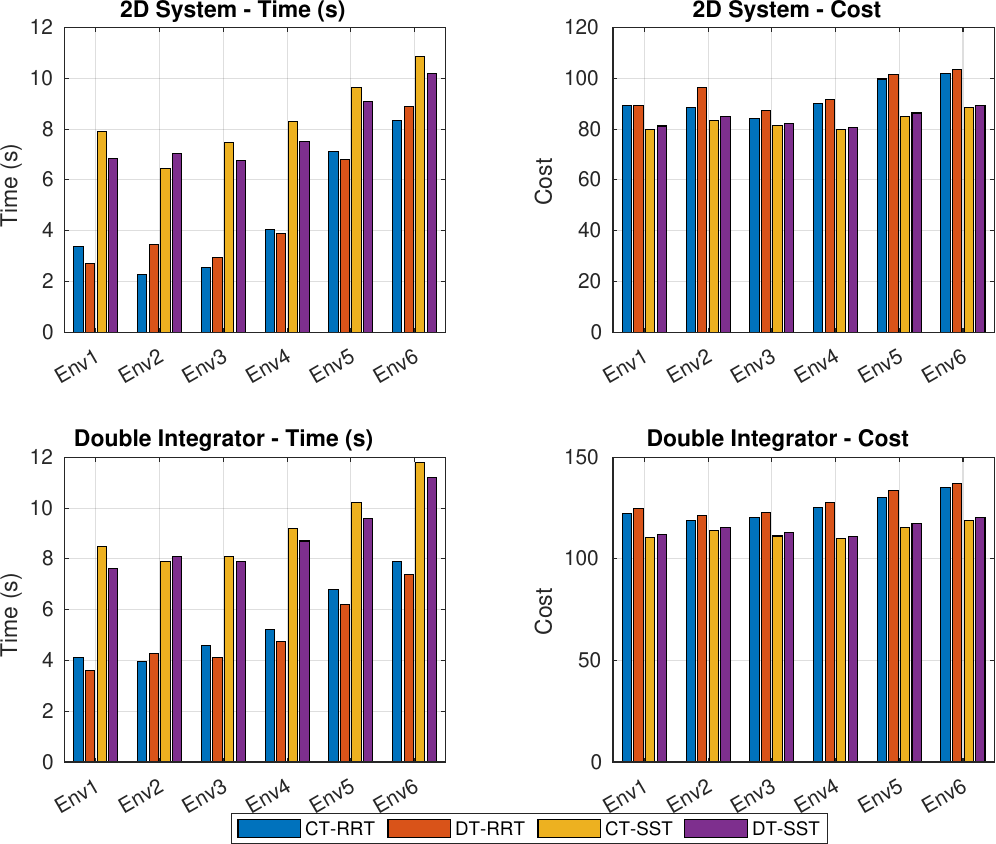}
    \caption{Comparing planning time and cost for \texttt{CT} vs. \texttt{DT} planners, with $U = [-1,1]^{2}$. The top and bottom plots depict results of the \texttt{RRT}- and \texttt{SST} planner, respectively.}
    \label{fig:barplots}
\end{figure}

\subsection{Monte-Carlo Simulations}
Next, to further emphasize the differences, we evaluate the planners under a Monte Carlo (MC) framework. In this setting, we generate 100 motion plans per environment and propagate the dynamics. For each roll-out, collision checking is performed along the resulting trajectory. We repeat this process for 5000 MC runs to obtain statistically significant estimates of performance. The resulting empirical probability of successfully avoiding obstacles and reaching the goal is computed accordingly.

Table~\ref{tab:mc_results} summarizes the results for both RRT- and SST-based implementations of CT-GBT and DT-GBT across four benchmark environments (Env. 3--6). The continuous-time variants achieve nearly perfect reliability, with success probabilities above 0.94 in all environments, including the most challenging narrow passage case. By contrast, the DT variants exhibit a sharp degradation in performance as the environment complexity increases: RRT drops from 0.65 in the empty case to only 0.06 in the most constrained environment, while SST shows a similar decline from 0.63 to 0.05. These results indicate that the CT formulation preserves feasibility under uncertainty and maintains robustness to noise, whereas the DT planner is prone to failure due to its optimistic treatment of dynamics and reduced accuracy in collision checking. Notably, the trends hold consistently across both RRT and SST planners, reinforcing the strength of the CT framework.

\subsection{Chance Constraint Evaluation}

Finally, to assess how well the planners respect probabilistic safety margins, we evaluate chance constraint violations under a similar MC framework. We propagate trajectories, sample at intermediate time instants between the discrete time steps and check whether the chance constraints are violated.  Table~\ref{tab:cc_results} reports the results across four benchmark environments (Env. 3--6). The CT formulations again demonstrate strong reliability, maintaining satisfaction probabilities in all cases. By contrast, the DT variants fail to enforce constraints as environment complexity increases, with satisfaction dropping to zero in the most constrained environment. These results highlight that CT-GBT not only preserves feasibility under uncertainty but also provides robust enforcement of chance constraints, while DT-GBT’s reduced accuracy leads to systematic violations.

\begin{table}[t!]
\centering
\caption{\small Empirical results for the probability of successfully avoiding obstacles and reaching the goal. 
}
   \label{tab:mc_results}
\begin{tabular}{@{}c|cccc@{}}
\toprule
 \texttt{Algorithm}  & $P_{s}$ Env. 3 & $P_{s}$ Env. 4 & $P_{s}$ Env. 5 & $P_{s}$ Env. 6 \\ \midrule
CT-GBT-RRT & \textbf{1.00}  & \textbf{1.00}  & \textbf{1.00}  & \textbf{0.95} \\
DT-GBT-RRT & 0.65 & 0.31 & 0.15 & 0.06\\
\midrule
CT-GBT-SST & \textbf{1.00}  & \textbf{1.00}  & \textbf{1.00} & \textbf{0.94}  \\
DT-GBT-SST & 0.63 &  0.22 & 0.12  & 0.05 \\ \bottomrule
\end{tabular}
\end{table}

\begin{table}[t!]
\centering
\caption{\small Empirical results for probability of chance-constraint satisfaction.
}
   \label{tab:cc_results}
\begin{tabular}{@{}c|cccc@{}}
\toprule
 \texttt{Algorithm} & $P_{s}$ Env. 3 & $P_{s}$ Env. 4 & $P_{s}$ Env. 5 & $P_{s}$ Env. 6 \\ \midrule
CT-GBT-RRT & \textbf{1.00}  & \textbf{0.98}  & \textbf{0.97}  & \textbf{0.91} \\
DT-GBT-RRT & 0.61 & 0.14 & 0.04 & 0.00\\
\midrule
CT-GBT-SST & \textbf{1.00}  & \textbf{0.98}  & \textbf{0.96} & \textbf{0.92}  \\
DT-GBT-SST & 0.59 & 0.11 & 0.03 & 0.00 \\ \bottomrule
\end{tabular}
    \vspace{-2mm}
\end{table}


\section{Conclusion}
\label{sec:conclusion}

This paper proposes a continuous-time Gaussian belief-space planning framework for sampling-based motion planning under process and measurement uncertainty. By modeling belief dynamics in continuous time and enforcing safety through belief barrier functions, the method reduces collision checking to efficient function evaluations while preserving algorithmic properties of the underlying planner. Results demonstrate that the continuous-time formulation achieves improved success rates and robust satisfaction of chance constraints across diverse benchmark environments compared to the discrete-time formulation, while being efficient. Overall, our methodology provides a principled foundation for safe and efficient continuous-time motion planning under uncertainty. Future work will extend this framework to richer task specifications such as signal temporal logic.


\bibliographystyle{IEEEtran}
\bibliography{bib, hybrid}

\end{document}